\documentclass[12p]{article}

\title{\LARGE \bf
Lifelong Localization in Dynamic Indoor Environments Combining Odometry with Sparse Distance Sampling
}

\usepackage[letterpaper,
            left=1.7in,
            right=1.7in,
            top=1.2in,
            bottom=1.2in]{geometry}

\author{Michael M. Bilevich$^{\dag}$ \and Tomer Buber$^{\dag}$ \and Dan Halperin$^{\dag}$}
\date{}

\usepackage{amssymb}
\usepackage{amsmath}
\usepackage{amsfonts}

\usepackage{amsthm}
\usepackage{listings}
\usepackage{hyperref}
\usepackage{graphicx}
\usepackage{wrapfig}
\usepackage{xcolor}
\usepackage{mathtools}
\usepackage{algorithm}
\usepackage{algpseudocode}
\usepackage{textcomp}
\usepackage{cite}
\usepackage{booktabs}

\usepackage{tikz}
\usetikzlibrary{shapes.geometric, arrows.meta, positioning, fit, backgrounds}

\newtheorem{theorem}{Theorem}[section]

\newtheorem{corollary}{Corollary}[section]

\newcommand{\bbR}{\mathbb{R}}
\newcommand{\bbS}{\mathbb{S}}
\newcommand{\bbP}{\mathbb{P}}
\newcommand{\calH}{\mathcal{H}}
\newcommand{\calQ}{\mathcal{Q}}
\newcommand{\calW}{\mathcal{W}}

\newcommand{\calC}{\mathcal{C}}

\newcommand{\calN}{\mathcal{N}}

\newcommand{\diam}{\mathrm{diam}}

\newcommand{\qbl}{q_{\mathrm{bl}}}
\newcommand{\qtr}{q_{\mathrm{tr}}}
\newcommand{\bel}{\mathrm{Bel}}

\begin{document}

\maketitle
\begingroup
\renewcommand\thefootnote{\fnsymbol{footnote}}
\footnotetext[2]{Blavatnik School of Computer Science and Artificial Intelligence, Tel-Aviv University, Israel. This work has been supported in part by the Israel Science
Foundation (grant no~3598/25), 
by the Blavatnik Computer Science Research Fund, 
and by the Shlomo Shmelzer Institute for 
Smart Transportation at Tel Aviv University.}
\endgroup
\thispagestyle{empty}
\pagestyle{empty}

\begin{abstract}

Localization is a key task in robot navigation, and many techniques exist for it.
In many plausible scenarios, a robot might face unforeseen, dynamic obstacles, rendering any pre-determined map inaccurate for localization.
In this work, we propose a robust lifelong localization framework in dynamic planar indoor environments, using the robot's odometry and sparse distance sampling.
We demonstrate how distance samples can be used to provide a robust prior on the robot's location. This technique can solve the kidnapped robot problem in real time, up to symmetries. Based on insights from real-world recorded data, we also account for dynamic obstacles.
We then fuse this prior, over time, with the odometry to converge to the robot's location.
A central property of our method is that it provably converges to the robot's ground truth pose even in large indoor environments when the environment is static. We further show that this guarantee also holds in dynamic environments, as long as the nature of those changes has been correctly learned.
We demonstrate the effectiveness of our approach in different real-world indoor environments. 
In particular, we achieve a localization comparable to SLAM with merely a few (sixteen) distance samples, as opposed to the full LiDAR range.
Sufficing with only sparse distance sampling is advantageous in terms of sensor cost, privacy, storage space, and transmission bandwidth.

\end{abstract}

\section{Introduction}

Robot localization is a critical ingredient in robot navigation~\cite{motroni_survey_2021}; even if the map of the environment is entirely known, and we plan a collision-free path of motion, we still need to determine the robot's location in the environment to carry out that path of motion. Determining the location of a robot in the environment is referred to as localization. Localization has been extensively studied and can be solved with various techniques and sensors. These methods may be intrinsic (attached to the robot) or extrinsic (attached to the environment).

Extrinsic approaches are commonplace; one well-known example is the Global Positioning System (GPS)~\cite{reina_adaptive_2007,vincent_comparison_2010,yousuf_sensor_2016}. By priorly placing landmarks in designated locations in the environment, the robot can triangulate its location with respect to the landmarks it perceives. These landmarks can be sophisticated and utilize properties of electromagnetic waves (such as RSSI~\cite{chen_survey_2022,yang_survey_2021} and RFID~\cite{motroni_survey_2021}), or they can be as simple as printed QR codes~\cite{bach_application_2023,morar_comprehensive_2020,nazemzadeh_indoor_2017}.
Methods originally developed for virtual reality applications can also be used for localization~\cite{greiff_performance_2019,hoppe_dronos_2019,taffanel_lighthouse_2021}.
A disadvantage of the extrinsic approach is that prior placement of landmarks may not be viable in some environments, such as confined environments (mines, caves)~\cite{debanne_global_1997,inostroza_robust_2023} or disaster-stricken environments (during fires or earthquakes)~\cite{nagatani_three-dimensional_2003,stormont_localization_2007}.

On the other hand, intrinsic localization is performed by sensors mounted on the robot itself. In such a case, we assume that we are given a map of the environment or that the sensors are sufficiently sophisticated to both map the environment and track the robot's location in the (simultaneously) mapped environment. The latter is the task of the intensively investigated Simultaneous Localization and Mapping (SLAM)~\cite{abaspur_kazerouni_survey_2022,aulinas_slam_2008,khan_comparative_2021,zhang_3d_2024}. Usually, the robot is mounted with depth cameras or LiDAR sensors. Both may be expensive~\cite{gerwen_indoor_2022}, especially when considering swarms of robots. Furthermore, if the computations are performed on a remote computer, transmitting LiDAR scans or image data over the network requires high bandwidth, which may cause communication interference~\cite{wang_localization_2018}.
One primary concern regarding using cameras attached to the robots is the potential breach of privacy~\cite{fernandes_detection_2016,lutz_robot_2020}.

\begin{figure}[t]
    \centering
    \includegraphics[width=0.85\linewidth]{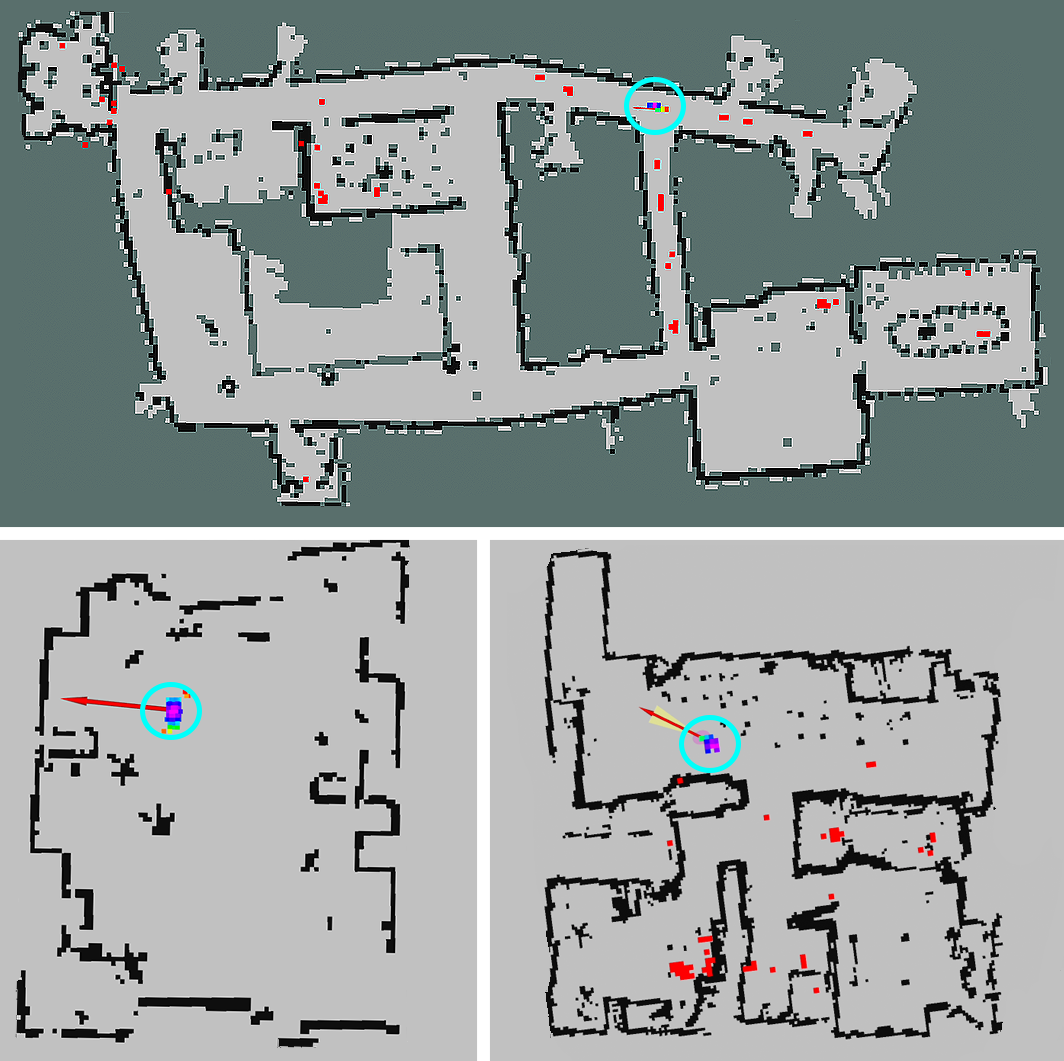}
    \vspace*{-0.25em}
    \caption{A demonstration of our algorithm in different real-world scenarios. All images are screenshots from RViz. In all scenarios, the robot uses merely $k=16$ distance measurements and its odometry to consistently find its location. All colorful squares are pose candidates returned by our method, where their color ranges from red (least likely) to magenta (most likely). Cyan marks the ground truth location. \textbf{Top}: A large-scale floor plan (\texttt{fl4}). \textbf{Bottom left}: A map of our laboratory (\texttt{lab446}). \textbf{Bottom right}: A floor plan of an apartment (\texttt{apt}).}
    \label{fig:main}
\end{figure}

The broad term \emph{localization} can refer to different problems, sometimes leading to vastly different solutions. One example is the \emph{kidnapped robot problem}~\cite{pak_state_2023}, where we need to determine the location of a robot given the map of the environment, but without any prior knowledge of its location. On the contrary, there is the \emph{tracking problem}~\cite{chen_survey_2022}, where we start with a known robot location and strive to fix any drift and measurement errors throughout its motion. This is also known as \emph{dead reckoning}~\cite{fuke_dead_1996,halvorsen_analysis_2021,do_dero_2024,guo_model-based_2024}. Life-long localization~\cite{tipaldi_lifelong_2013,muhlfellner_summary_2016,zhao_general_2021,sousa_systematic_2023,li_ll-localizer_2025} is the problem of determining the localization of a robot in a known environment throughout multiple sessions. Usually, this is performed by storing and updating the map~\cite{sousa_systematic_2023,li_ll-localizer_2025} during multiple sessions. Life-long indoor localization poses many difficult challenges, such as overcoming the presence of dynamic obstacles and changes in the environment~\cite{fang_liloc_2025,chen_slam-ramu_2024}, as well as computational~\cite{sousa_systematic_2023} challenges of effectively and quickly maintaining and refining the maps.

In this work, we present a lifelong indoor localization technique that requires merely odometry and a sparse distance sampling. We will also refer to this method as \emph{sparse distance sampling localization (SDSL)}, since we use only a few distance samples, as opposed to a range of distance samples returned by, for example, a LiDAR sensor.

Recent works~\cite{mustafa_guaranteed_2018,guyonneau_guaranteed_2014,song_guaranteed_2025,bilevich_sensor_2023,bilevich_localization_2024,bilevich_indoor_2025} have dealt with intrinsic localization in already-known environments. These methods are deterministic, unlike more classical approaches such as Bayesian filters and Monte-Carlo localization~\cite{chen_bayesian_2003,yang_survey_2021}. These works are robust for different kinds of errors, such as errors in measurements and mapping. We note that~\cite{mustafa_guaranteed_2018,guyonneau_guaranteed_2014,song_guaranteed_2025} used the full range of distance measurements given by a LiDAR sensor, while our works~\cite{bilevich_sensor_2023,bilevich_localization_2024,bilevich_indoor_2025} used only a sparse distance sampling.

We improve upon~\cite{bilevich_sensor_2023,bilevich_indoor_2025} by incorporating a probabilistic machinery, inspired by Bayesian localization, to utilize the robot's odometry. We also account for unforeseen dynamic obstacles and changes in the environment, based on the insight presented in~\cite{bilevich_localization_2024}.

We assume that the environment is a subset of $\bbR^2$, and our robot is \emph{planar} with three degrees of freedom - two for translations in the environment, and one for rotation about its origin. The robot is equipped with an array of range sensors with known offsets from the origin and any odometry sensor (e.g., an inertial measurement unit or IMU, wheel encoder, or optical flow). We assume that we are given a map that is a good approximation of the environment, albeit it may have some topological errors and unforeseen dynamic obstacles in the environment.

Similar to previous works, this framework is beneficial in terms of cost, privacy, and transmission bandwidth.

See Figure~\ref{fig:main} for an illustration of our method.

\begin{figure}[t!]
\centering
\includegraphics[width=0.995\linewidth]{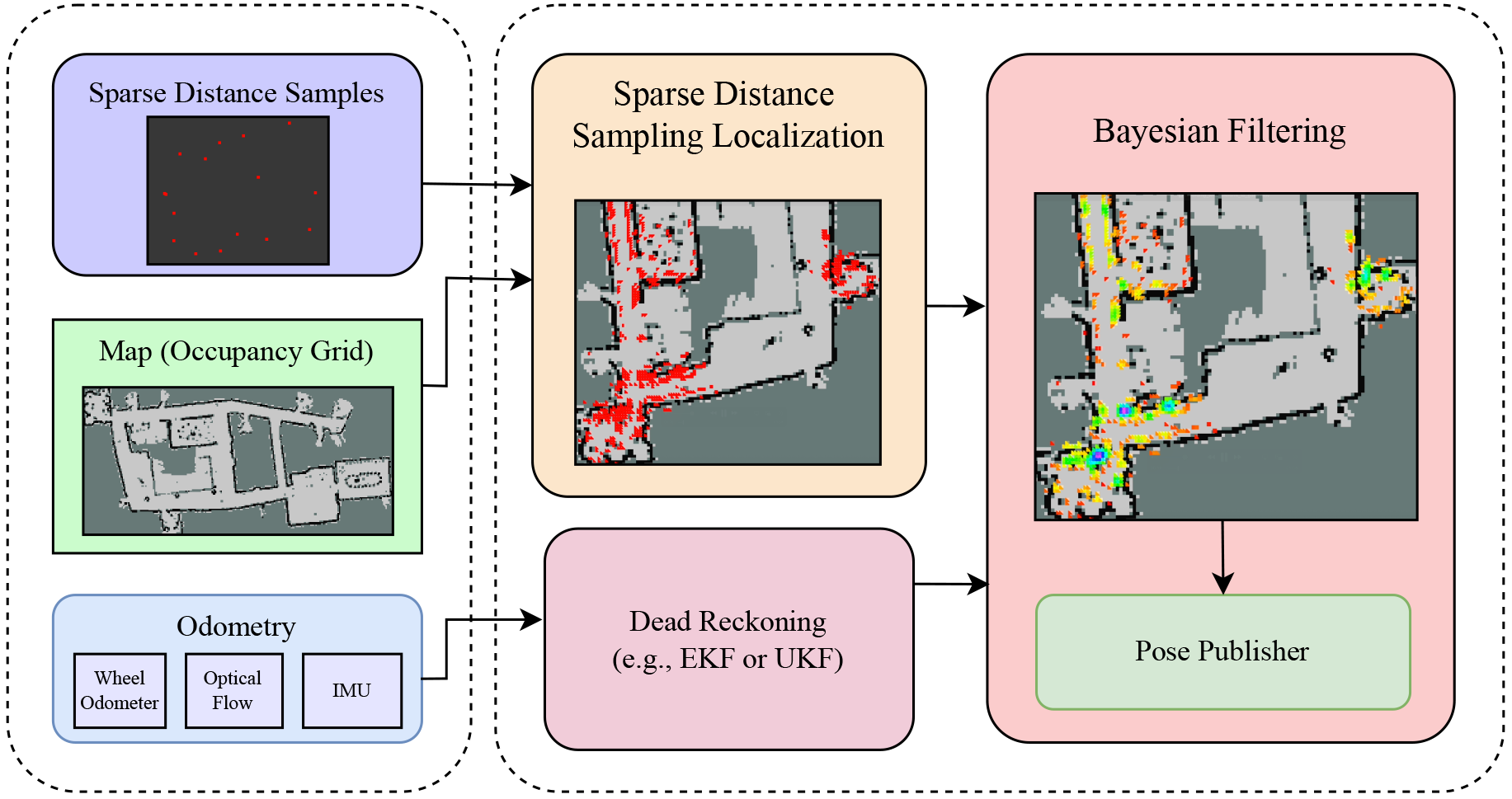}
\label{fig:sdsl-overview}
\vspace*{-1.25em}
\caption{An overview of our method. On the left-hand side are the inputs to our method; we assume that the map is pre-determined, and we receive a sparse (e.g., $k=16$) distance sampling, and any kind of odometry. We use odometry to calculate the dead reckoning of the robot. We also match the distance samples (SDSL) with the map to get the set of all possible pose candidates. We then fuse the dead reckoning with the SDSL to get the belief state for our robot, which is the likelihood of each pose candidate. Once there is one pose that is more likely than the others, we publish it as the robot's location.}
\end{figure}

\subsection{Contribution}

Our contributions are as follows:

\begin{itemize}
    \item An effective lifelong localization technique for planar robots, fusing odometry dead reckoning with sparse distance sampling, which is inherently robust for the kidnapped robot problem even 
    in dynamic environments.
    \item Guarantees on the correctness and completeness of the algorithm, proving that our localization technique would always report a pose candidate close to the true location, when the dynamic nature of the environment is reasonably predictable.
    \item We demonstrate our method on a mobile planar robot, showing the potential for real-world application.
    \item An open-source C++ library with Python bindings~\cite{jakob_nanobind_2022} that can utilize multi-core CPUs, as well as ROS2~\cite{macenski_robot_2022} packages for easy deployment on general robots.
\end{itemize}

\section{Preliminaries and Problem Statement}

\subsection{Subdivision Search of Fibers}

A fiber of some function $f$, denoted by $f^{-1}(y)$, is the set of all values $x$ such that $f(x)=y$, given a value $y$. A key observation of~\cite{bilevich_sensor_2023,bilevich_indoor_2025} is that given the distance measurements, the set of all possible locations is the intersection of fibers of the distance function. This task is studied in computer graphics, computational geometry, and computer algebra~\cite{botsch_polygon_2010}. One common example for the search for fibers is the meshing of implicit surfaces, which can be carried out by a variety of popular techniques: grid-based methods like marching cubes and dual contouring~\cite{chen_neural_2021,lorensen_marching_1998,newman_survey_2006,nielson_dual_2004} or Delaunay refinement~\cite{boissonnat_meshing_2006,boissonnat_delaunay_2014,boissonnat_anisotropic_2015}. Another common technique of searching and representing implicitly defined geometry is by subdivision searches, i.e., via quadtrees for two dimensions, octrees for three dimensions, and orthrees for higher dimensions~\cite{campolattaro_quadtrees_2024,meagher_geometric_1982,samet_overview_1988}. The method presented in this work is also inspired by methods of the soft subdivision search (SSS) framework~\cite{hsu_rods_2019,wang_soft_2013,yap_soft_2015}. Generally, subdivision methods are superior to grid-based methods in terms of running time complexity~\cite{bilevich_note_2025}.

\subsection{Problem Statement}
\label{ssec:problem-statement}
In this work, we assume that the environment $\calW \subset \bbR^2$ is a closed subset of the plane. The robot can translate and rotate freely in the environment, and its configuration space~\cite{lavalle_planning_2006} is $\calC \coloneqq \bbR^2\times \bbS^1$. For a configuration $q=(q_p, q_\theta)\in \calC=\bbR^2\times\bbS^1$, $q_p$ is the position of the robot's origin in the plane and $q_\theta$ is its orientation.

Note that the orientation space can be represented by the angles $\bbS^1 \simeq [0,2\pi)$ or the corresponding points on the unit circle $\bbS^1 \subset \bbR^2$. In this work, we use both notions interchangeably. 

The robot is mounted with $k$ range sensors, with known offsets $g_1,\dots,g_k \in \bbR^2 \times \bbS^1$; Initially, the sensors are mounted at an offset ${g_i}_p$ from the robot's origin and point in the direction ${g_i}_\theta$.

Assuming that the robot is at some configuration $q\in \calC$, when measuring distance from the $i$th sensor, we cast a ray emanating from the sensor's current position, which is the offset ${g_i}_p$ translated and rotated by $q$. The ray's direction is the direction ${g_i}_\theta$ rotated by $q$. We denote the application of the pose $q$ on the sensor position ${g_i}_p$ and direction ${g_i}_\theta$ by $q\cdot {g_i}_p$ and $q\cdot {g_i}_\theta$, respectively.

Let $h:\bbR^2\times \bbS^1 \to \bbR_{\geq 0}$ denote the distance measurement function. Then $h(p,\theta)$ is the distance between the point $p\in\bbR^2$ and the first intersection of a ray emanating from $p$ in the direction $\theta\in\bbS^1$ with the boundary of the environment $\partial\calW$. The distance measurement of the $i$th sensor, denoted by $d_i\geq 0$, is $d_i = h(q\cdot {g_i}_p, q\cdot {g_i}_\theta)$.

\paragraph*{Problem Statement}{
Given a map of the environment $\calW$, the sensor offsets $g_i$, and their corresponding distance measurements $d_i$ for $i=1,\dots,k$, find the set of all poses $q\in\calC$ that, by placing the robot at pose $q$, it would measure the distances $d_1,\dots,d_k$ at the offsets $g_1,\dots,g_k$, respectively.
}

\section{Method Overview}

We now provide an overview of our method. See Figure~\ref{fig:sdsl-overview} for an illustration of the proposed approach. 
In the figure, we use ROS2~\cite{macenski_robot_2022} notions throughout, and ROS was our choice for the implementation.
However, the method and algorithms can be implemented with any other framework. 
Furthermore, as described in Section~\ref{ssec:implementation-details}, our core method is implemented as a standalone C++ library independent of ROS.

\subsection{Sensors Input and Pre-determined Map}
 Our method builds on three main inputs. 
The first is a pre-existing map of the environment, which provides the global frame of reference. 
The second is the robot's odometry, giving a continuous estimate of its motion over time. 
The third is a sparse set of $k$ range measurements, for which the offsets $g_1,
\dots,g_k$ are known and pre-determined.
No additional sensing or framework-specific assumptions are required.

\subsection{Message Processing and Sensor Fusion}
    Our approach works by first processing each modality (the odometry, dead reckoning, and the SDSL) separately and efficiently. Then, we combine (fuse) both intermediate results into a more robust localization.
    Using the robot's odometry, we can compute dead reckoning. That is, by integrating the odometry, we can estimate the trajectory/pose of the robot with respect to its origin of motion~\cite{do_dero_2024}. This estimate suffices locally, but may drift for longer horizons due to factors like wheel slip and sensor noise.
    Note that this origin is not necessarily the origin of the map. Importantly, this local motion frame is not necessarily aligned with the global map frame, so an additional transformation between the two must eventually be established.

    With sparse distance sampling, our method produces a set of candidate poses that capture all robot locations consistent with the measurements in the global map. 
    
Since the number of samples is small and the environment may include moving obstacles, the resulting set of poses can spread over vastly different areas of the map. These candidate poses are similar to the particles in Monte Carlo localization. However, they are computed deterministically, based on the geometric constraints posed by the distance measurements. Details of the sparse distance sampling localization technique are presented in Section~\ref{sec:sdsl-method}

We then fuse the odometry trajectory with the poses set returned by the SDSL method. 
When using both, the odometry provides smooth local motion, and the sparse samples keep the estimate tied to the global map. 
The result is a weighted set of poses that reflects the robot’s most likely location. 
Details of the fusion process are given in Section~\ref{sec:fusion}.

\subsection{Method Output}

Finally, we must choose the robot's most probable (single) location. To do so, we cluster the particles and weight the probability/likelihood of each cluster. Whenever a cluster is more likely than the rest, its center of mass (weighed by the poses' likelihood) is chosen as a pose estimate, and the transform from the \texttt{odom} frame to the \texttt{map} is updated accordingly.

\section{SDSL for Kidnapped Robot Problem}
\label{sec:sdsl-method}

In this section, we describe the Sparse Distance Sampling Localization (SDSL) technique, which outputs the set of all feasible robot poses, given a single sparse distance sampling, i.e., a set of $k$ measurements $d_1,\dots,d_k \geq 0$ and their corresponding pre-determined offsets $g_1, \dots,g_k\in \bbR^2\times\bbS^1$.

To do so, we perform a subdivision search, similar to the one described in detail in~\cite{bilevich_note_2025,bilevich_indoor_2025}. The derivations are in particular identical to~\cite{bilevich_indoor_2025}, and thus are omitted here for brevity. We also note that, by definition (see Section~\ref{ssec:problem-statement}), for a single distance measurement $d_i$, the set of all possible locations that satisfy that measurement is a \emph{fiber} of the distance measurement function.

We choose a parameter $\delta > 0$, which is our desired approximation precision, and we also assume we are given $\varepsilon > 0$ as a distance measurement error bound. We start with a bounding volume $V_0$ of the configuration space---as one voxel, and recursively subdivide a voxel into smaller voxels only if it may contain the ground truth result, until all voxels are of diameter $<
\delta$. Thus, the crux of our method is effectively and correctly determining whether a voxel may or may not contain a ground truth result.

Furthermore, improving upon~\cite{bilevich_indoor_2025}, we assume that there is a \emph{$k$-$k'$-dynamic gap}~\cite{bilevich_localization_2024}. That is, we assume that out of the $k$ distance measurements, at least $k'$ samples correspond to features in the pre-determined map. Of course, this $k$-$k'$ gap depends on the dynamic nature of the environment, the semantic context of the present time (e.g., holidays or nights may have less traffic of dynamic obstacles), and even the location in the room; one possible example is that poses close to walls are almost surely guaranteed to have $k' > k/2$, regardless of the dynamic nature. See Figure~\ref{fig:half-kprime} for an illustration.

\begin{figure}[t!]
    \centering
    \includegraphics[width=0.3\linewidth]{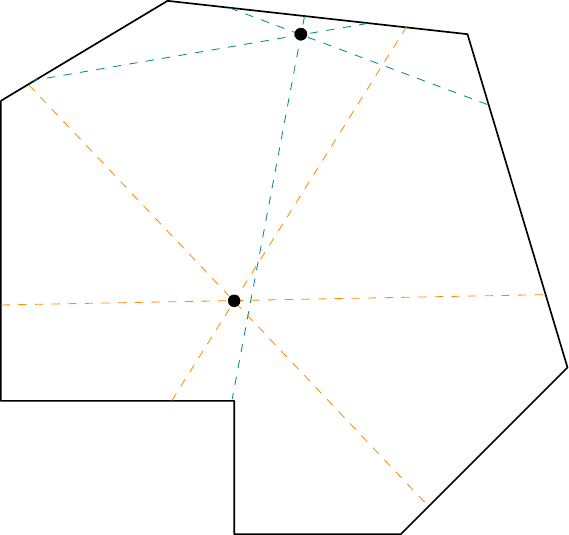}
    \caption{Illustration of two poses with $k=6$ distance measurements. The top point (whose rays are cyan) is closer to the environment's wall than the bottom (whose rays are orange). Notice how the bottom point is more likely to see dynamic obstacles, as there is more space between the robot and the environment. For the top point, about half of its rays measure the wall, which is relatively close.}
    \label{fig:half-kprime}
\end{figure}

Algorithm~\ref{alg:sdsl} describes our overall method. 
For each voxel, we estimate the value $k'$. The method $\mathrm{VoxelXPred}$ (for Voxel intersection Prediction),
tests whether a single measurement $d_i$ and its corresponding offset $g_i$ ``agree'' with the environment. That is, whether there exists some pose in the voxel for which, by placing the robot at that pose, the sensor whose initial offset is $g_i$ would measure the distance $d_i$ (up to an error of $\varepsilon$).
If at least $k'$ measures do not agree with the environment, the voxel is discarded; otherwise, the voxel is split (or reported, if its diameter is less than $\delta$).

To evaluate $\mathrm{VoxelXPred}$, we merely intersect an axis-aligned box with the boundary of the environment $\calW$.
Denote by $F_{d,g}:\calC \to \bbR^2$ the function that takes a robot configuration $q$, and returns the point in $\bbR^2$ which is $d$-units forward from the sensors $g$ when the robot is at pose $q$.
Note that if the intersection $F_{d,g}(V)\cap\partial\calW \neq \emptyset$ is not empty, there exists some pose in that voxel $V$ that, by placing the robot in that pose, the sensor at offset $g$ would measure exactly $d$. 
Although the geometry of the set $F_{d,g}(V)$ can be precisely described and computed, we approximate it by taking its bounding box, which is simpler to compute.
More details can be found in~\cite{bilevich_indoor_2025},
and the pseudo-code is presented in Algorithm~\ref{alg:voxel-isect-pred}.

The method $\mathrm{Split}(V)$ splits a voxel into a collection of smaller voxels such that their union equals the original voxel $V$. 
In this work, we split the voxel into $2^3$ sub-voxels by splitting into two in each dimension.
The method $\mathrm{EstimateKPrime}(V)$ is described in Section~\ref{ssec:estimating-kprime}.

Finally, we coincide the voxel outputs with a point cloud of the voxel centers. This can be done by choosing a value of $\delta$ which is strictly smaller than the desired accuracy, such that each two poses in the same voxel are effectively indistinguishable. 

\begin{algorithm}
    \caption{Sparse Distance Sampling Localization}
    \begin{algorithmic}
        \Require $\calW\subset\bbR^2, d_1,\dots,d_k\geq 0, g_1,\dots,g_k \in \bbR^2\times\bbS^1$
        \Require $\delta > 0, \varepsilon > 0$
        \Require $V_0 \subset \calC, \calW\times\bbS^1\subseteq V_0$
        \State $\calQ \gets \{V_0\}$, $\calQ' \gets \emptyset$
        \While {$\max_{V\in\calQ} \diam V \geq \delta$}
            \For {$V\in \calQ$}
            \State $k' \gets \mathrm{EstimateKPrime}(V)$
            \State $cnt \gets 0$
            \For {$i=1,\dots,k$}
                \If {$\mathrm{VoxelXPred(\calW, d_i, g_i, V, \varepsilon)}$}
                    \State $cnt \gets cnt + 1$
                \EndIf
            \EndFor
            \If{$cnt \geq k'$}
                $\calQ' \gets \calQ' \cup \mathrm{Split}(V)$
            \EndIf
            \EndFor
            \State $\calQ \gets \calQ'$, $\calQ' \gets \emptyset$
        \EndWhile\\
        \Return $\calQ$
    \end{algorithmic}
    \label{alg:sdsl}
\end{algorithm}

\begin{algorithm}
    \caption{$\mathrm{VoxelXPred}(\calW, d, g, V, \varepsilon)$}
    \begin{algorithmic}
        \Require $V=[\qbl,\qtr)\subset \calC$, $\calW\subset \bbR^2, g\in \bbR^2\times\bbS^1, d\geq 0, \varepsilon > 0$
        \State $\alpha_1 \gets g_x + d \cdot \cos g_\theta$, $\alpha_2 \gets g_y + d \cdot \sin g_\theta$
        \State $\beta_1 \gets \tan^{-1}(-\alpha_2 /\alpha_1)$, $\beta_2 \gets \tan^{-1}(\alpha_1/\alpha_2)$
        \State $\theta_1\gets {\qbl}_\theta$, $\theta_2 \gets {\qtr}_\theta$,  $\Theta \gets \{\theta_1, \theta_2\}$
        \For{$j=-3,\dots,3$, $i=1,2$}\\
            \Comment{In fact, only few $j$ values are feasible}
            \If{$\theta_1\leq \beta_i+j\cdot\pi \leq \theta_2$}
                \State $\Theta \gets \Theta \cup \{\beta_i+j\cdot \pi\}$
            \EndIf
        \EndFor
        \State $x_1 \gets \infty, x_2 \gets -\infty$, $y_1 \gets \infty, y_2 \gets -\infty$
        \For{$\theta \in \Theta$}
            \State $\gamma_1 \gets \cos\theta\cdot\alpha_1 - \sin\theta\cdot \alpha_2$,  $\gamma_2 \gets \sin\theta\cdot\alpha_1 + \cos\theta\cdot \alpha_2$
            \State $x_1\gets\min\{x_1,\gamma_1\}$, $x_2\gets\max\{x_2,\gamma_1\}$
            \State $y_1\gets\min\{y_1,\gamma_2\}$, $y_2\gets\max\{y_2,\gamma_2\}$
        \EndFor
        \State $x_1 \gets x_1 +{\qbl}_x-\varepsilon$, $x_2 \gets x_2 +{\qtr}_x+\varepsilon$
        \State $y_1 \gets y_1 +{\qbl}_y-\varepsilon$, $y_2 \gets y_2 +{\qtr}_y+\varepsilon$
        \State $B \gets \big[(x_1,y_1), (x_2,y_2)\big)\subset\bbR^2$
        \Return $B \cap \calW \neq \emptyset$
    \end{algorithmic}
    \label{alg:voxel-isect-pred}
\end{algorithm}

\subsection{\texorpdfstring{Estimating the Expected Value of $k'$}{Estimating the Expected Value of k'}}
\label{ssec:estimating-kprime}

Given a voxel $V$, we need to provide a reasonable estimate of the value of $k'$, so we will not miss the ground truth location. However, we need a value that is still close enough to $k$ so that the resulting set $\calQ$ will be sufficiently small. 

As stated in Section~\ref{sec:sdsl-method}, this value of $k'$ may depend on various factors. In this work, we suggest a heuristic, which looks only at the distance $d\left((q_x, q_y), \partial\calW\right)\geq 0$ of a configuration $q = (q_x,q_y,q_\theta)\in\calC$ from the boundary of the environment $\partial\calW$. We note that the distance from the boundary is an invariant feature under translations and rotations of the map, and should be at least comparable for environments whose rooms are of similar dimensions.

Based on data collection from real-world environments, which is described in Section~\ref{ssec:exp-ekprime}, we implement the following heuristic function $\tilde{f}_{k'}:[0,\infty) \to [0,1]$ which returns the expected \emph{ratio} $k'/k$ of $k'$ over $k$ for a pose whose distance from the boundary is $\mathrm{dist} = d\left((q_x, q_y), \partial\calW\right)\geq 0$: 

\begin{align}
    \tilde{f}_{k'}\left(\mathrm{dist}\right) = \begin{cases}
        0.8 & \text{if } \mathrm{dist} < 0.975[m]\;,\\
        0.7 & \text{else}
    \end{cases}\;.
\end{align}
As shown in Section~\ref{ssec:exp-ekprime}, this simple function is a good rough approximation for typical office environments, and we see that it also performs well in practice.

To evaluate $\tilde{f}_{k'}$ on a voxel, we take the minimal value of $\tilde{f}_{k'}$ of each of its vertices and its center point.

\section{Fusing SDSL with Odometry}
\label{sec:fusion}

In this section, we describe how the lifelong dead reckoning pose estimate can be combined with the point cloud returned by the SDSL algorithm (Algorithm~\ref{alg:sdsl}), using Bayesian filtering~\cite{fox_bayesian_2003}.

Denote by $\hat{\gamma}:[0,\infty)\to\calC$ the estimated path, returned by any dead reckoning localization technique. This $\hat{\gamma}$ is estimated from the robot's odometry. We assume that for each time $t$, we can query the value of $\hat{\gamma}(t)$.

We expect a lifelong stream of consecutive runs of the SDSL algorithm, i.e., a sequence $X_1, X_2,\dots$ where each $X_n=\{q_1^{(n)}, \dots,q_{N_n}^{(n)}\}\subset \calC$ is a set of $N_n$ configurations, returned by Algorithm~\ref{alg:sdsl} at time $t_n$.
\smallskip
\noindent
For clarity, we introduce the following notation:
$q_i^{(n)}$ denotes the $i$-th candidate pose in $X_n$.
Denote the following estimated pose difference:
\begin{align}
    U_n = \hat{\gamma}(t_n) \cdot \left(\hat{\gamma}(t_{n-1})\right)^{-1}\;,
\end{align}
which is the estimated transformation the robot performed from time $t_{n-1}$ to $t_n$, and can be thought of as simply the robot's odometry. Similar to Monte Carlo localization~\cite{thrun_robust_2001}, we can use these odometries $U_1,U_2,\dots$ to evaluate the likelihood, over time, of each pose. The main difference is that instead of stochastically resampling, at each iteration $n$, we recompute the set of all possible poses, using the geometric constraints imposed by the perceived distance measurements, via the SDSL method. In Section~\ref{sec:analysis}, we also show that this approach has rigorous theoretical merits.

\smallskip

In Algorithm~\ref{alg:fusion} we present the pseudo-code for the procedure that computes the \emph{belief} state, denoted by $\bel(\cdot)$, which computes for each pose $q_{i}^{(n)}$ its likelihood $\bel(q_{i}^{(n)})$. 

From the recursive formulation of Bayesian filtering, we recall that the belief state can be recursively enumerated~\cite{fox_bayesian_2003}:

\begin{align}
    \bel(X_n) \propto \int_{X_{n-1}} \bbP[X_n | X_{n-1}=q, U_n] \cdot \bel(q)dq\;.
    \label{eq:bel}
\end{align}

Note that since we took into account the distance measurements in the SDSL step, we ignore them here. Equation~\ref{eq:bel} omits the denominator that usually appears in the Bayes filter recursion, as we can simply normalize the probabilities at each time step $n$.

At time $n=1$, the belief state is set as a uniform distribution, and all poses have the same likelihood.
We model the transition $\bbP[X_n | U_n, X_{n-1}=q]$ as a normal distribution, with a standard deviation of $\varepsilon$. That is, $X_n \sim \calN(U_n\cdot X_{n-1}, \varepsilon^2)$.

The method $\mathrm{Normalize(a_1,\dots,a_m)}$ divides each element by the sum of all elements, i.e., $a_i \gets a_i/\sum_{j=1}^ma_j$. 

We also note that instead of directly computing the exponential Gaussian probability, we can instead use the \emph{log likelihoods}~\cite{chen_bayesian_2003}. For simplicity, Algorithm~\ref{alg:fusion} shows the straightforward approach.

\begin{algorithm}
    \caption{Fusing SDSL with Odometry}
    \begin{algorithmic}
        \Require $X_1,X_2,\dots \subset \calC$, $U_1,U_2,\dots \in \calC$
        \State $p \gets \{\}$
        \For{$i=1,\dots,N_1$}
            $p(q_i^{(1)})=\frac{1}{N_1}$
        \EndFor
        \For {$n=2,\dots$}
            \For{$q_i^{(n)}=q_1^{(n)},\dots,q_{N_n}^{(n)}\in X_n$}
                \State $\bel(q_i^{(n)})\gets 0$
                \For{$q_j^{(n-1)}=q_1^{(n-1)},\dots,q_{N_{n-1}}^{(n-1)}\in X_{n-1}$}
                    \State $s\gets \frac{1}{\sqrt{2\pi \cdot \varepsilon^2}}\cdot e^{-\frac{1}{2\varepsilon^2}||q_i^{(n)} - U_n\cdot q_j^{(n-1)}||^2}$
                    \State $\bel(q_i^{(n)})\gets \bel(q_i^{(n)}) + s\cdot \bel(q_j^{(n-1)}) $
                \EndFor
            \EndFor
            \State $\mathrm{Normalize}(\bel(q_1^{(n)}), \dots, \bel(q_{N_n}^{(n)}))$\\
            \Comment{Report $\bel(q_1^{(n)}),\dots,\bel(q_{N_n}^{(n)})$}
        \EndFor
    \end{algorithmic}
    \label{alg:fusion}
\end{algorithm}

\section{Analysis and Guarantees}
\label{sec:analysis}
In this section, we present some analysis and guarantees of our method.
We first note that using the same proof of~\cite{bilevich_indoor_2025}, our algorithm is \emph{output sensitive}. That is, the time complexity is a function of the resulting set of possible locations for the robot.

\begin{theorem}
    \label{thm:runtime-simple}
    Assume that each call to $\mathrm{VoxelXPredicate}$ and $\mathrm{EstimateKPrime}$ takes at most $\calQ_{\calW}$ time,
    and there are $m$ reported poses at the last iteration of the algorithm. Then Algorithm~\ref{alg:sdsl} runs in time
    \begin{align}
        \Theta(k\cdot \calQ_{\calW}\cdot m\log\delta)\;.
    \end{align}
\end{theorem}

To further understand the magnitude of the number of reported poses, we can use the notion of \emph{Hausdorff measure and dimension}, which is an extension of the Lebesgue measure for subsets of the Euclidean space that have Lebesgue measure zero, but still have some non-zero ``area'' in a lower dimension. For example, the Hausdorff dimension of a $2$-manifold embedded in $\bbR^3$ is $2$, and its Hausdorff measure is its area~\cite{evans_measure_2018}.

\begin{theorem}
    \label{thm:runtime-hausdorff}
    Assume that $M$ is the intersection of all fibers corresponding to distance measurements. Then, by taking $\delta\to0$, Algorithm~\ref{alg:sdsl} converges to $M$. Furthermore, the time complexity of our method is
    \begin{align}
        \Theta\left(k\cdot \calQ_{\calW}\cdot \log\delta^{-1}\cdot(\delta_0\cdot\delta^{-1})^{\calH_{\dim}(M)}\cdot \calH^{\calH_{\dim}(M)}(M)\right)\,
    \end{align}
    where $\delta_0$ is the diameter of $V_0$, $\calH_{\dim}(M)$ is the Hausdorff dimension of $M$ and $\calH^{\calH_{\dim}(M)}(M)$ is the $\calH_{\dim}(M)$-dimensional Hausdorff measure of $M$.
\end{theorem}

\begin{proof}
    Proof for both theorems follows the same as~\cite{bilevich_indoor_2025}.
\end{proof}

We also note that, assuming that our estimate of $k'$ is indeed correct, our method is robust:

\begin{theorem}
    \label{thm:robust}
    Assuming that the $k'$ estimate value of $\mathrm{EstimateKPrime}(V)$ is a correct lower bound to the actual number of perceived dynamic obstacles for each voxel $V$, Algorithm~\ref{alg:sdsl} is guaranteed to output at least one pose $q$ which is $\delta$-close to the ground truth location of the robot. 
\end{theorem}

An immediate corollary of Theorem~\ref{thm:robust} is that our method is inherently robust to kidnappings---even if the odometry is significantly off, e.g., the robot was picked up and placed in a different room, the new set of samples $X_n$ will already contain the new location of the robot. Notably, this robustness is achieved without the need for any explicit kidnap-detection heuristics or re-sampling~\cite{dellaert_monte_1999}.
This is particularly useful in real-world settings, where unexpected obstacles, blocked sensors, or sudden disruptions can easily throw off the robot’s pose. Our method handles these situations naturally, recovering from errors without special reset routines or manual intervention.

Finally, we note the following comparison with Monte Carlo-like methods:

\begin{corollary}
    When sampling random configurations in the bounding volume $V_0$ whose diameter is $\delta_0$, to guarantee, in expectation, that we will have at least one sample that is $\delta$-close to the ground truth location, we need to query 
    \begin{align}
       O(\delta_0\cdot \log\delta^{-1}\cdot \delta^{-\dim\calC})
    \end{align}
    particles.

    However, if the fiber intersection $M$ is a set of $m$ distinct poses, each contained in a ball of radius proportional to $\delta$, i.e., $\alpha\cdot \delta$, then our method would query only
    \begin{align}
        O(m\cdot V_{\dim\calC}\cdot \alpha^{\dim\calC}\cdot\log\delta^{-1})\;,
    \end{align}
    where $V_{\dim\calC}$ is the volume of the unit $\dim\calC$-ball.

    In other words, assuming that the query time of a particle and a voxel are similar, we substitute the exponential dependence $\delta^{-\dim\calC}$ with $\alpha^{\dim\calC}$, which is a significant saving for a sufficiently small $\alpha$.
\end{corollary}

\begin{proof}
    First, note that the upper bound $O(\delta_0\cdot \log\delta^{-1}\cdot \delta^{-\dim\calC})$ on random sampling is a classical result~\cite{janson_random_1986}.

    We note that the Hausdorff dimension of a collection of $m$ $\dim\calC$-balls is $\dim\calC$, and their Hausdorff measure is $m$ times the volume of each sphere, times its radius to the $\dim\calC$-th power.

    Overall, using Theorem~\ref{thm:runtime-hausdorff}, we get an upper bound of (we omit the $k\cdot\calQ_{\calW}$ as analyze the number of elements to query):
    \begin{align}
        O\left(\log\delta^{-1}\cdot (\delta_0\cdot\delta^{-1})^{\dim\calC}\cdot m\cdot V_{\dim\calC} \cdot (\alpha\cdot\delta)^{\dim\calC}\right)\;.
    \end{align}
\end{proof}

\section{Experiments and Results}

\subsection{Implementation Details}
\label{ssec:implementation-details}

Our open source software\footnote{\url{https://www.cgl.cs.tau.ac.il/?p=7143}} is an independent C++ header-only library with Python bindings in \emph{nanobind}~\cite{jakob_nanobind_2022}, and ROS2~\cite{macenski_robot_2022} packages for the algorithm and for the physical robot. 
We use OpenMP~\cite{openmp_architecture_review_board_openmp_2008} for running the voxel intersection predicate in parallel. We use CGAL AABB Tree~\cite{alliez_2d_2024} for fast axis-aligned bounding-box intersection queries.

\begin{figure}[t]
    \centering
    \includegraphics[width=0.8\linewidth]{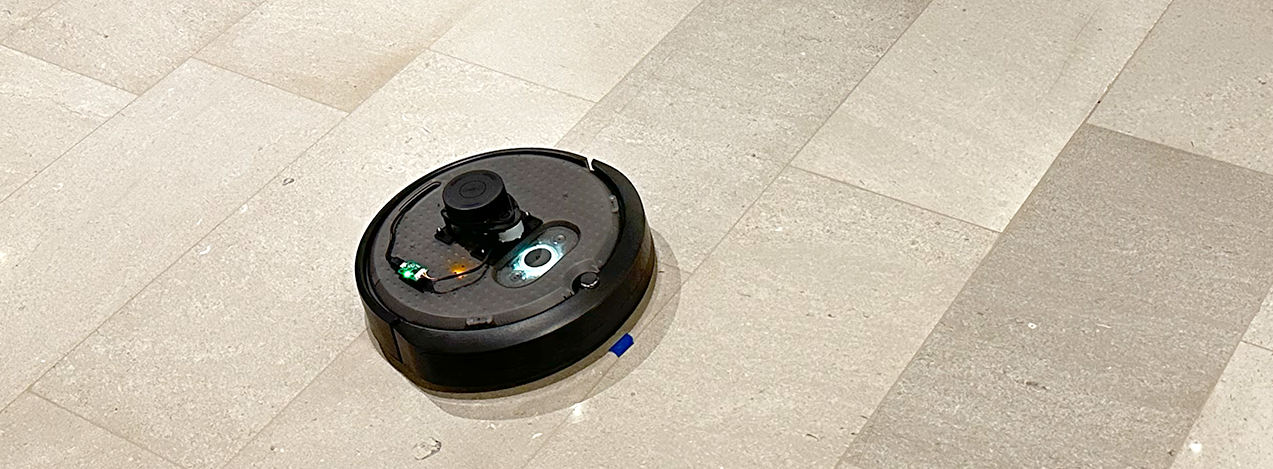}
    \caption{An image of our iRobot Create 3 in a corridor of map \texttt{fl4}. The RPLIDAR A1 is mounted on top of the robot.}
    \label{fig:tony}
\end{figure}

We run our experiments on an iRobot Create 3 mobile robot (See Figure~\ref{fig:tony}), equipped with a SLAMTEC RPLIDAR A1 two-dimensional LiDAR, and a Raspberry Pi 4, running ROS2 Humble. Algorithm~\ref{alg:sdsl} was run on a laptop with an 8-Core CPU.

\subsection{Evaluation in Real-World Scenarios}
\label{ssec:exp-eval}

We demonstrate our method on a physical robot in the following scenarios: 
\begin{itemize}
    \item \texttt{lab446}: Our laboratory, which is about $\sim 20[m^2]$.
    \item \texttt{fl4}: A full floor plan, which is about $\sim 600[m^2]$.
    \item \texttt{apt}: An apartment, which is about $\sim 70[m^2]$.
\end{itemize}

The maps were acquired using \texttt{slam-toolbox}~\cite{macenski_slam_2021}.
All environments are dynamic in nature; in our \texttt{lab446} and in our floor \texttt{fl4}, furniture (chairs, doors) is moved, as well as people naturally walking nearby. The maps for both environments were acquired a month before the experiments. In the \texttt{apt} environment, the map was acquired a few days before the experiments. During those days, furniture was moved and rearranged slightly, due to natural, typical usage.

\smallskip

In each environment, the robot completes $5$ laps. Each lap passes through $3$ landmarks whose exact locations are pre-determined. At each lap, we measure the deviation $\epsilon_{\mathrm{pos}}$ of position in meters, $\epsilon_{\mathrm{rot}}$ of orientation in radians, of the approximated location of the robot on each of the pre-determined landmarks.

We benchmark our method against the \texttt{nav2-amcl}~\cite{macenski_marathon_2020} ROS package, which is a commonly used Monte Carlo localization tool, based on Adaptive Monte Carlo Localization~\cite{dellaert_monte_1999}.

In Figure~\ref{fig:chosen-landmarks}, we show the three chosen landmarks for each environment, overlaid with our belief state and the pose estimate returned by \texttt{nav2-amcl}. In the supplementary video, we showcase one lap in each environment.

In Table~\ref{tab:error-rates} we present the average errors across all scenarios, for both methods. Note that our method achieves performance comparable to state-of-the-art tools, outperforms them for some cases, and does not drift over time.

\begin{table}[h!]
\caption{Average error rates across all scenarios for both methods}
\centering
\begin{tabular}{@{}lcccccc@{}}
\toprule
 & \multicolumn{2}{c}{\textbf{\texttt{lab446}}} & \multicolumn{2}{c}{\textbf{\texttt{fl4}}} & \multicolumn{2}{c}{\textbf{\texttt{apt}}} \\ 
\cmidrule(lr){2-3} \cmidrule(lr){4-5} \cmidrule(lr){6-7}
 & Ours & \texttt{amcl} & Ours & \texttt{amcl} & Ours & \texttt{amcl} \\ 
\midrule
$\epsilon_{\mathrm{pos}}\ [m]$ & 0.039 & \textbf{0.034} & \textbf{0.095} & 0.219 & 0.091 & \textbf{0.066} \\ 
$\epsilon_{\mathrm{rot}}\ [\mathrm{rad}]$  & \textbf{0.04} & 0.159 & \textbf{0.083} & 0.43 & \textbf{0.0247} & 0.184 \\ 
\bottomrule
\end{tabular}
\label{tab:error-rates}
\end{table}

\begin{figure}[t!]
    \centering
    \includegraphics[width=0.995\linewidth]{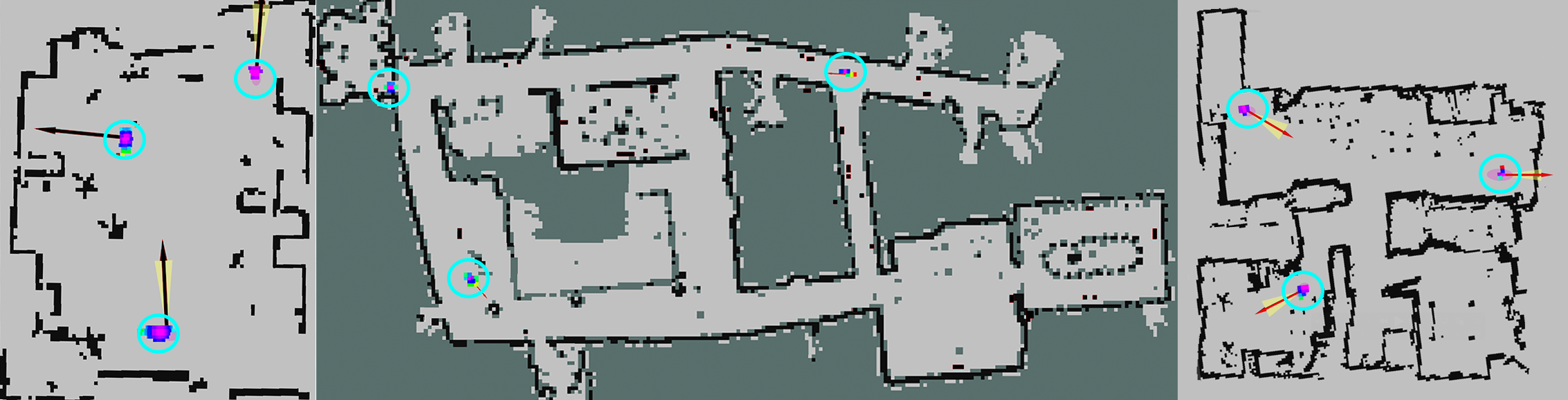}
    \vspace*{-1.25em}
    \caption{Visualization of the three chosen landmarks in each environment, circled in cyan. From left to right: \texttt{fl4}, \texttt{lab446}, and \texttt{apt}. Similar to Figure~\ref{fig:main}, the colorful squares represent the pose estimate in each landmark. We overlay our pose estimate with the estimate returned by \texttt{nav2-amcl}. For clarity, all other pose candidates are omitted.}
    \label{fig:chosen-landmarks}
\end{figure}

\subsection{Robustness Against Kidnapped Scenarios}
\label{ssec:exp-kidnap}

Recall that we solve the kidnapped robot problem from scratch at each iteration. Hence, we are able to quickly re-localize when the robot is kidnapped, without manually detecting the occurrence of such events. 

For each environment, we start from a known robot location, which is precise for both our method and \texttt{nav2-amcl}. We then pick up and place the robot at $10$ different locations in the environment. Furthermore, whenever possible, some furniture was moved, and at each location, there were people in the vicinity of the robot, posing as dynamic obstacles. 

Immediately when the robot is placed down, we request a global re-localization from \texttt{nav2-amcl}. Our method carries on uninterrupted.
For each method, we measure the time it takes to converge to a single pose.
To help both methods differentiate between symmetric poses, until convergence, the robot moves slowly forward.

In \texttt{lab446}, our method converged on all ten cases immediately at the first iteration when placed down, taking less than one second. On the other hand, \texttt{nav2-amcl} converged only on $7$ out of the ten placements, taking on average $59.04[\sec]$ and at least half a minute for all cases. 

In \texttt{apt}, our method still had a convergence rate of $100\%$, taking $12.3[\sec]$ on average. The method \texttt{nav2-amcl} converged only on $2$ out of the ten placements, taking on average $56.8[\sec]$.

Finally, we note that on \texttt{fl4} the method \texttt{nav2-amcl} did not converge at all for any of the placements. Our method had a convergence rate of $100\%$, taking $15.7[\sec]$ on average.

\subsection{Measuring the Expected Value of \texorpdfstring{$k'$}{k'}}
\label{ssec:exp-ekprime}

In Section~\ref{ssec:estimating-kprime}, we discuss the need for approximating the value of $k'$. In lieu of imposing arbitrary assumptions on the dynamic obstacles, we collect data and learn such an approximation.

We have placed five SLAMTEC RPLIDAR S1 two-dimensional LiDARs in two environments: (i) $\texttt{lab446}$ and the kitchenette of \texttt{fl4}, which recorded point clouds over a month and a week, respectively. The LiDARs are placed in opposing corners of each environment, so they can jointly map the entire room when there are no dynamic obstacles.

After recording, we take the union of the point clouds for each timestamp. By sampling random poses in the environment, and casting $k$ rays in simulation, we can evaluate the value of $k'$ that a robot placed in that random pose would perceive. We sampled for different values of $k=4,\dots,20$. As suggested in Section~\ref{ssec:estimating-kprime}, for each sampled point we also computed its distance from the boundary of the environment $\partial\calW$.

The processed result is a table which maps values of $\mathrm{dist}=d\left((q_x, q_y), \partial\calW\right)$ to the $k'/k$ ratio that would have been perceived at the time. We use the table for $\texttt{lab446}$ to fit a decision stump~\cite{abaspur_kazerouni_survey_2022}.
We achieved the decision rule: \emph{if $\mathrm{{dist} < 0.975[m]}$ then the $k'/k$ ratio is $0.8116$, else $0.7542$}. In Section~\ref{ssec:estimating-kprime}, we round down those ratios to get different integral values when multiplying this ratio by $k=16$.

We then interdependently evaluated a decision stump for the table corresponding to the kitchenette of $\texttt{fl4}$, and achieved the decision rule: \emph{if $\mathrm{dist}<0.75[m]$ then the $k'/k$ ratio is $0.8086$, else $0.8279$}. 

Note that the decision rules behave differently, although similar, probably due to the geometry and the different nature of the two environments. Still, we note that our suggested $\tilde{f}_{k'}(\mathrm{dist})$ is a lower bound for both, and performed well in all scenarios tested.

\section{Conclusions and Future Work}

In this work, we demonstrate an effective method of lifelong indoor localization using merely a sparse distance sampling and the robot's odometry. While this method achieves results comparable to state-of-the-art tools, it excels in inherent robustness against significant sudden shifts in the robot's location, as we solve efficiently the kidnapped robot at every iteration, from scratch. We also consider the existence of dynamic, unforeseen changes in the pre-determined map, based on insights gathered from real-world data.

However, when the robot faces many dynamic obstacles in extreme cases, we may temporarily lose its location. Once the dynamic obstacles pass, the robot will be able to quickly re-localize, but we plan to improve upon this shortcoming in future work.
We also plan to improve this localization technique further by actively controlling the robot and choosing a good intermediate motion that would allow faster convergence to the ground truth location.

Furthermore, in this work, we used a simple heuristic to learn the environment's dynamic nature. Nevertheless, more advanced machine learning tools can be incorporated to achieve even more accurate and robust results.

%
%
\bibliographystyle{IEEEtran}
\bibliography{bibliography}

\end{document}